\documentclass{alt2019}

\usepackage{bm}

\newtheorem{thm}{Theorem}
\newtheorem{prop}[thm]{Proposition}
\newtheorem{cor}[thm]{Corollary}
\newtheorem{defn}[thm]{Definition}

\newcommand{\R}{\mathbb{R}}

\newcommand{\sC}{{\mathcal C}}

\newcommand{\sH}{{\mathcal H}}

\newcommand{\sX}{{\mathcal X}}

\renewcommand{\a}{\alpha}
\renewcommand{\b}{\beta}
\newcommand{\g}{\gamma}
\newcommand{\eps}{\epsilon}

\newcommand{\ind}[1]{{\bm 1}_{\{#1\}}}

\newcommand{\bcase}{\left\{ \begin{array}{ll} }
\newcommand{\ecase}{\end{array} \right. }
\newcommand{\ev}{\mathbb{E}}

\newcommand{\etat}{\tilde{\eta}}
\newcommand{\etahat}{\widehat{\eta}}
\newcommand{\fhat}{\widehat{f}}
\newcommand{\that}{\widehat{t}}
\newcommand{\Fhat}{\widehat{F}}
\newcommand{\Ghat}{\widehat{G}}
\newcommand{\Lhat}{\widehat{L}}
\newcommand{\Qhat}{\widehat{Q}}

\newcommand{\Yt}{\tilde{Y}}

\title[Optimal Domain Adaptation]{A Generalized Neyman-Pearson Criterion for Optimal Domain Adaptation}
\usepackage{times}

 \coltauthor{\Name{Clayton Scott} \Email{clayscot@umich.edu}\\
 \addr Division of Electrical and Computer Engineering and Department of Statistics\\
       University of Michigan\\
       Ann Arbor, MI 48109 USA
 }

\begin{document}

\maketitle

\begin{abstract}
In the problem of domain adaptation for binary classification, the learner is presented with labeled examples from a source domain, and must correctly classify unlabeled examples from a target domain, which may differ from the source. Previous work on this problem has assumed that the performance measure of interest is the expected value of some loss function. We study a Neyman-Pearson-like criterion and argue that, for this optimality criterion, stronger domain adaptation results are possible than what has previously been established. In particular, we study a class of domain adaptation problems that generalizes both the covariate shift assumption and a model for feature-dependent label noise, and establish optimal classification on the target domain despite not having access to labelled data from this domain.
\end{abstract}


\section{Introduction}

In the problem of domain adaptation for binary classification, the learner is given labeled examples from a source distribution, and must design a classifier that performs well on a potentially different target distribution. We consider the semi-supervised setting where, in addition to labeled training data from the source distribution, the learner has access to an unlabeled sample from the target distribution. To gain traction on this problem, it is necessary to make some assumptions relating the source and target distributions, and several types of assumptions have been considered previously in the literature, such as covariate shift, target shift, and various forms of label noise.

Previous work on domain adaptation has focused almost exclusively on a particular class of performance measures, namely, those expressible as the expected value of some loss function, with particular attention being paid to the 0-1 loss. We argue that the difficulty of a domain adaptation problem depends on the performance measure being optimized, and the focus on loss-based criteria has limited the contributions of prior work. The present work was motivated by the problem of classification with feature- (or instance-) dependent label noise (FDLN), where previous efforts to minimize the expected 0-1 loss (probability of error) require excessively strong assumptions on the nature of the label noise. Our work also bears on the covariate shift model, where prior work requires source and target distributions to be rather similar in order to make strong performance guarantees.

We examine an optimality criterion for binary classification that we call the controlled discovery rate (CDR), which is a special case of a more general class of generalized Neyman-Pearson criteria. We show that it is possible to optimize CDR over a broad class of domain adaptation problems that we refer to as {\em covariate shift with posterior drift}. We do this by showing that the CDR criterion is {\em immune} to this class of domain adaptation problems, meaning one can train a classifier as if the source and target distributions were the same, and still optimize the CDR criterion when they are different. Thus, no particularly novel algorithms are required to achieve optimal domain adaptation. Our results lead to more general statements of optimality for covariate shift and FDLN than have previously been established.

\subsection{Notation}
\label{sec:notation}

Let $\sX \subset \R^d$ denote the feature space and $\{0,1\}$ the label space. Let $Q$ be a probability distribution on $\sX \times \{0,1\}$. If the pair $(X,Y)$ are jointly distributed according to $Q$, let $Q_y$, $y\in\{0,1\}$, denote the conditional distribution of $X$ given $Y=y$. $Q_0$ and $Q_1$ are referred to as the ``class-conditional distributions." Denote by $\pi_Q$ the marginal probability that $Y=1$, and by $\eta_Q(x)$ the conditional probability that $Y=1$ given $X=x$. In classification, $Y$ may be viewed as an unknown parameter that must be predicted from $X$, and in this spirit we refer to $\pi_Q$ and $\eta_Q(x)$ as the ``prior" and ``posterior" probabilities associated to $Q$. Finally, let $Q_X := \pi_Q Q_1 + (1-\pi_Q) Q_0$ be the marginal distribution of $X$.


Throughout this work we assume that $Q_0$ and $Q_1$ have densities $q_0$ and $q_1$, defined w.r.t. some dominating measure $\mu$, and related to $\eta_Q(x)$ via Bayes rule:
\begin{equation}
\label{eqn:bayes}
\eta_Q(x) = \frac{\pi_Q q_1(x)}{(1-\pi_Q)q_0(x) + \pi_Q q_1(x)}.
\end{equation}
We will often refer to a second distribution $P$ on $\sX \times \{0,1\}$ in addition to $Q$. The associated quantities $P_0, P_1, \pi_P, \eta_P, P_X, p_0$ and $p_1$ are defined analogously. In this case the densities $p_0, p_1, q_0, q_1$ are assumed to have a common dominating measure. The choice $\mu = P_0 + P_1 + Q_0 + Q_1 $ is always valid, but typically $\mu$ is either the Lebesgue or counting measure.

\subsection{Objective}
\label{sec:objective}

\sloppy In domain adaptation there are two distributions, $P$ and $Q$, referred to as the {\em source} and {\em target} distributions. We consider the semi-supervised setting where the learner observes $(X_1, Y_1), \ldots, (X_m,Y_m) \sim P$ and $X_{m+1}, \ldots, X_{m+n} \sim Q_X$, and must design a classifier whose performance/optimality is assessed with respect to $Q$. The focus of this paper is to consider a particular optimality criterion, the CDR criterion, such that optimal classification is possible under a class of domain adaptation problems now described.

\subsection{Covariate Shift with Posterior Drift}

The class of domain adaptation problems considered is a combination of two fundamental classes that have been separately considered in prior work. The first, {\em covariate shift}, assumes
\begin{description}
\item[(CS)] $\eta_P = \eta_Q$.
\end{description}
In particular, under \textbf{(CS)}, the source and target posteriors are the same, while $P_X$ and $Q_X$ are allowed to differ. Covariate shift has been studied extensively and related work is discussed in Section \ref{sec:related}. It arises, for instance, when there is a {\em sample selection bias} that causes source and target feature vectors to follow different distributions \citep{heckman77bias}. For example, in developing a classifier for a certain disease, source subjects may have volunteered for a clinical study, while testing subjects are drawn from the general public. These two populations are different and hence $P_X \ne Q_X$, but presumably $\eta_P = \eta_Q$.

The second type of domain adaptation, which we call {\em posterior drift}, assumes
\begin{description}
\item[(PD)] $P_X = Q_X$, and there exists a strictly increasing function $\phi$ such that for all $x$, $\eta_P(x) = \phi(\eta_Q(x))$.
\end{description}
Posterior drift is a model for FDLN. In this work, label noise refers to a corruption of the labels of the training data, and is in addition to any uncertainty in the optimal label arising from overlap of $Q_0$ and $Q_1$. Posterior drift may be viewed as a model for so-called ``annotator" noise, which models the way a human might (noisily) assign labels to unlabeled data \citep{urner12}. In particular, let $(X,Y,\Yt)$ be jointly distributed. Let $Q$ be the distribution of $(X,Y)$, where $X$ is the feature vector and $Y$ the true label. Let $P$ be the distribution of $(X,\Yt)$, where $\Yt$ is a noisy label assigned by the annotator. Clearly $P_X = Q_X$ in this setting. Furthermore, $\eta_Q$ is the true probabilistic labeller, while $\eta_P$ is the probabilistic labeller associated to the annotator. \textbf{(PD)} asserts that as the probability of the true label being 1 increases, so too does the probability of the annotator's label being 1. See Section \ref{sec:related} for more discussion of FDLN.

Finally, it is natural to combine these two assumptions, leading to the following.
\begin{description}
\item[(CSPD)] There exists a strictly increasing function $\phi$ such that for all $x$, $\eta_P(x) = \phi(\eta_Q(x))$.
\end{description}
In this model, the marginal distribution of $X$ is allowed to shift, as in \textbf{(CS)}, while the posterior is simultaneously allowed to drift, as in \textbf{(PD)}.



\subsection{Contributions}

To our knowledge, this work is the first to study the \textbf{(CSPD)} class of domain adaptation problems, making it the largest class of domain adaptation problems for which immunity (and hence optimal performance) has been established. Relative to prior work on covariate shift, we are the first to establish optimal domain adaptation without requiring a high degree of similarity between $P$ and $Q$ (see related work below). Relative to prior work on classification with FDLN,  our work is the first to establish optimal performance without overly restrictive assumptions on the label noise (again, see related work). We also introduce a new family of optimality criteria that has not previously been considered in machine learning. Finally, we introduce two algorithms for optimizing CDR in the semi-supervised setting, including the first analysis of a level set method based on kernel logistic regression.

\subsection{Outline}

In the next section we introduce a family of generalized Neyman-Pearson criteria for binary classification. Section \ref{sec:related} discusses related work. In Section \ref{sec:cdr}, consistent estimators for the CDR criterion are established, and in Section \ref{sec:da}, we synthesize the results of prior sections to explain how optimal domain adaptation is achieved under covariate shift with posterior drift. The final section concludes, and proofs are found in an appendix.

\section{A Generalized Neyman-Pearson Criterion}
\label{sec:gnp}

We introduce a family of constrained criteria for classifier design, indexed by
parameters $0 \le \theta_0 < \theta_1 \le 1$ and $0 \le \alpha \le 1$, and defined with respect to a distribution $Q$ as described in Section \ref{sec:notation}. The Neyman-Pearson (NP) criterion corresponds to the special case $\theta_1 = 1$ and $\theta_0 = 0$. After this section, we will be particularly interested in the case $\theta_1 = 1$ and $\theta_0 = \pi_Q$ in the context of the domain adaptation problems mentioned previously.

A classifier is a function $g: \sX \to [0,1]$. We view classifiers as potentially randomized, where $x$ is classified as 1 with probability $g(x)$, independent of all other random variables. The {\em power} $B_Q(g)$ of a classifier $g$ is the probability that the predicted label is 1, given that the true label is one. That is,
$$
B_Q(g) := \ev_{Q_1}[g(X)] = \int g(x) q_1(x) d\mu(x).
$$
The power is also referred to as 1 - Type II error, detection rate, true positive rate, sensitivity, or recall.
The {\em size} $A_Q(g)$ of a classifier $g$ is the probability that a predicted label is 1, given that the true label is zero. That is,
$$
A_Q(g) := \ev_{Q_0}[g(X)] = \int g(x) q_0(x) d\mu(x).
$$
Size is also known as the Type I error, false alarm rate, false positive rate, or 1 - specificity.

For the {\em generalized Neyman-Pearson} (GNP) criterion with parameters $0 \le \theta_0 < \theta_1 \le 1$ and $0 < \alpha < 1$, a classifier $g$ is optimal if it solves the following optimization problem:
\begin{align}
\max_g \ & \theta_1 B_Q(g) + (1-\theta_1) A_Q(g) \label{opt:con} \\
\mbox{s.t.} \ & \theta_0 B_Q(g) + (1-\theta_0) A_Q(g) \le \a \nonumber.
\end{align}
where the max is over all classifiers. Notice that $B_Q$ is an accuracy measure, whereas $A_Q$ is an error quantity. The condition $\theta_0 < \theta_1$ ensures that the relative emphasis on acruracy in the objective, and error in the constraint, lead to a meaningful criterion for classification.
Indeed, the optimal classifier is obtained by thresholding $\eta_Q(x)$. Equivalently, the optimal classifier is a likelihood ratio test (LRT), since 
$\eta_Q(x)$ and $q_1(x) / q_0(x)$ are monotonically related according to
\eqref{eqn:bayes}. 
\begin{thm}
\label{thm:con}
Given $0 \le \theta_0 < \theta_1 \le 1$, and $0 < \a \le 1$,
there exist $t_{Q,\a} \in [0,1]$, $q_{Q,\a} \in [0,1)$, such that a solution to \eqref{opt:con} is
$$
g_{Q,\a}(x) := \bcase
1, & \eta_Q(x) > t_{Q,\a}, \\
q_{Q,\a}, & \eta_Q(x) = t_{Q,\a}, \\
0, & \eta_Q(x) < t_{Q,\a}.
\ecase
$$
\end{thm}
The proof uses an argument of \citet{blanchard16ejs} to show that the GNP criterion can be viewed as a conventional NP criterion with respect to two different contaminated versions of $Q$. Then, the NP lemma is used to show that the optimal classifier is a LRT, and this result is transformed back to the GNP criterion.

In this paper we are primarily concerned with the special case where $\theta_1 = 1$ and $\theta_0 = \pi_Q$. The expression in the constraint becomes $D_Q(g) := Q_X(g(X) = 1)$, which we refer to as the {\em discovery rate} of $g$. In this case, we aim to solve
\begin{align*}
\max_g \ & B_Q(g)  \\
\mbox{s.t.} \ & D_Q(g) \le \a,
\end{align*}
which yields the most powerful classifier that predicts at most a fraction $\a$ of test instances as positive. We refer to this specific criterion as the {\em controlled discovery rate} (CDR) criterion. The CDR criterion is desirable in applications where positively classified examples from the target domain will be subjected to further scrutiny, and there is a limited budget to conduct follow-up investigations. 
For example, in information retrieval it is common that only the top $100\a \%$ of the test instances will be inspected by a user. In this context, the CDR criterion seeks the classifier with maximum recall that assigns a positive label to $100\a \%$ of the test instances. Thus, CDR is similar in spirit to criteria that aim to measure ``accuracy at the top" \citep{boyd12nips}. Previous work relating to the CDR criterion is discussed in the next section.

We show in this work that the CDR criterion can be optimally learned under \textbf{(CSPD)}. The intuition behind this fact, and the {\em primary insight} of this paper, is as follows. Consider the infinite sample setting where $P$ and $Q_X$ are known. Since $P$ is known, we know $\eta_P$, which is  monotonically equivalent to $\eta_Q$ under \textbf{(CSPD)}. By this monotone equivalence, the optimal classifier (for the target domain) has the form $g(x) = \ind{\eta_P(x) \ge t}$ for some $t$. This threshold $t$ can be set to ensure that $D_Q(g) = \a$ (which must be satisfied by the optimal classifier) because $D_Q(g)$ {\em depends on $Q$ only through} $Q_X$. In the finite sample case, our algorithms naturally rely on estimates of $\eta_P$ and $D_Q$. The details of this argument are worked out in the sequel.


\section{Related Work}
\label{sec:related}

\textbf{Target Shift}: A kind of dual of covariate shift is {\em target shift}, where $P_0 = Q_0$ and $P_1 = Q_1$, but $\pi_P \ne \pi_Q$. This form of domain adaptation arises frequently in applications where training and testing data are gathered according to different sampling plans. For example, training data gathered prospectively may have a user-determined $\pi_P$, while testing data analyzed retrospectively may have a $\pi_Q$ that is beyond the user's control.

Target shift is a class of problems that satisfy neither \textbf{(CS)} nor \textbf{(PD)}, but do satisfy \textbf{(CSPD)}. To see this, just note that $P_X = \pi_P P_1 + (1-\pi_P) P_0 \ne  \pi_Q P_1 + (1-\pi_Q) P_0 = Q_X$, so \textbf{(PD)} is violated, and
$$
\eta_P(x) = \frac1{1 + \frac{1-\pi_P}{\pi_P}\frac{p_0(x)}{p_1(x)}} \ne \frac1{1 + \frac{1-\pi_Q}{\pi_Q}\frac{p_0(x)}{p_1(x)}}
= \eta_Q(x),
$$
so \textbf{(CS)} is violated. Yet clearly $\eta_P$ and $\eta_Q$ are monotonically equivalent, so \textbf{(CSPD)} holds.

Previous work on target shift has focused on estimating $\pi_Q$ in the semi-supervised setting \citep{hall81, titterington83, saerens01, plessis12, sanderson14}. Since target shift is a special case of \textbf{(CSPD)}, our methods optimize the CDR criterion for such problems, notably {\em without} needing to estimate $\pi_Q$. In fact, all GNP criteria are immune to target shift.

\textbf{Immunity}: 
An optimality criterion is immune to a class of domain adaptation problems if the optimal classifier is the same for both the source distribution and the target distribution (see Appendix A for a more formal definition). Practically speaking, immunity implies that the learner can ignore the possibility of domain adaptation (i.e., assume $P=Q$) and still be optimal when $P \ne Q$. As an example, consider the probability of error as a performance measure (i.e., the risk with 0-1 loss). It is well known that the probability of error is immune to symmetric, feature-independent label noise \citep{angluin88,kearns93,jabbari10}. To see this, suppose $Q$ is the ``clean" distribution on $(X,Y)$, and $P$ is the contaminated distribution on $(X,\Yt)$, such that a realization of $(X,\Yt)$ is obtained by drawing $(X,Y)$ from $Q$, and replacing $Y$ with $1-Y$ with probability $\nu < \frac12$, independent of $X$. It follows that $\eta_P(x) = (1-\nu)\eta_Q(x) + \nu (1-\eta_Q(x))$. This implies $\eta_P(x) -\frac12 = (1-2 \nu) (\eta_Q(x) - \frac12)$, and therefore the optimal classifiers for $P$ and $Q$ coincide. Thus, training a classifier to optimize probability of error on noisy training data leads to an optimal classifier with respect to $Q$.

Immunity has been established for other types of label noise. The probability of error is immune to symmetric, feature-dependent label noise, while the AUC is immune to a type of feature-dependent annotator noise that implies \textbf{(PD)} \citep{menon18}. The balanced error rate (BER) is immune to asymmetric label-dependent (but feature-independent) label noise \citep{menon15icml}. \citet{menon15icml} also argue that BER is the only performance measure that is immune to label-dependent label noise. The class of performance measures they study does not include the CDR criterion, so there is no contradiction with our results which apply to label-dependent label noise (see below).

Other instances of the GNP family also possess immunity for certain domain adaptation problems. For example, consider the target shift problem described above. Any GNP criterion is trivially immune to target shift (when trained only on labeled training data from the source distribution) because it does not depend on the prior class probability in the first place. The same is obviously true for other criteria that don't involve the class priors, such as the balanced error rate or the min-max criterion. The Neyman-Pearson criterion has further been shown to be immune to classification with one-sided, label-dependent label noise, also known as learning with positive and unlabeled examples \citep{blanchard10}. In Appendix A we argue that any GNP criterion with $\theta_0 = 0$ is immune to one-sided, feature-dependent label noise. The immunity of NP for target shift has been described by \citet{xia18}.

In this work we show that, in the semi-supervised setting, the CDR criterion is immune to \textbf{(CSPD)}. To our knowledge, this is the most general class of problems for which immunity has been established for some binary classification optimality criterion. For further discussion of immunity, see Appendix A.

\textbf{Covariate Shift and General Domain Adaptation}: Previous work on covariate shift \citep{shimodaira00} has focused on performance measures that can be expressed as risks, that is, as the expectation of a loss function with respect to $P$ or $Q$. Because of this, many papers have focused on the problem of estimating the ratio $q_X(x)/p_X(x)$, where $q_X$ and $p_X$ are the densities of $Q_X$ and $P_X$, respectively \citep{zadrozny04, huang07, cortes08, sugiyama08kliep, bickel09, kanamori09}. Unfortunately, this introduces an intermediate (and potentially quite challenging) estimation problem into the learning pipeline. In contrast, learning with respect to the CDR criterion avoids estimation of the density ratio.

Several previous works have theoretically studied, under covariate shift as well as more general domain adaptation settings, when a good classifier on the target domain can be learned. For example, several papers have shown that the target risk can be bounded in terms of the source risk and some notion of ``discrepancy" between $P$ and $Q$ (and possibly other terms) \citep{bendavid07nips,bendavid10ml,blitzer08nips,mansour09colt,cortes15kdd,germain16icml}, which has led to the conclusion that in order ``for generalization to be possible . . .  $Q$ and $P$ must not be too dissimilar" \citep{mansour09colt}. \citet{bendavid12alt} argue that covariate shift alone is insufficient to ensure good performance on the target domain. In particular, they argue that under covariate shift, good performance on the target domain cannot be guaranteed even if the supports are equal and densities $q_X$ and $p_X$ are mutually bounded.

In the present work, we show that optimal domain adaptation is possible assuming that \textbf{(CSPD)} holds, that the support of $P_X$ contains the support of $Q_X$, and two relatively benign nonparametric conditions. In particular, optimal domain adaptation is possible even though $P_X$ and $Q_X$ (and hence $P$ and $Q$) might be vastly different. Our results are not incompatible with previous results because the settings are somewhat different. First, as mentioned previously, we consider a different optimality criterion. Second, our focus is statistical consistency, whereas previous work often considers a fixed hypothesis space. Third, our analysis concerns the error of a classifier {\em relative} to the best possible classifier, whereas some previous work has addressed making the risk small in an {\em absolute} sense. 

\textbf{Classification with Feature-Dependent Label Noise}:
Classification with label noise is a form of domain adaptation, although it has not always been described as such. In this setting, $(X,Y,\Yt)$ are jointly distributed. $Q$ is the distribution of $(X,Y)$, where $Y$ is the true label of $X$, and $P$ is the distribution of $(X,\Yt)$, where $\Yt$ is a corrupted version of $Y$. We reiterate that in this discussion, label noise is in addition to any uncertainty in the optimal label arising from overlap of the supports of $Q_0$ and $Q_1$.

In the case of label-dependent label noise (LDLN), the probability that a training label is flipped depends only on the true label.  The label-dependent case is fairly well understood \citep{blanchard16ejs, natarajan18jmlr, rooyen18jmlr} in the two-class setting. In essence, the difference between the source and target domains can be reduced to two parameters, $\rho_i := \Pr(\Yt \ne i \mid Y = i)$, $i\in\{0,1\}$, the label flip probabilities for each class. Given knowledge of these proportions (which can be estimated), it is not difficult to modify a learning algorithm to successfully adapt to the target domain. We also note that LDLN is a special case of \textbf{(PD)} provided $\rho_0 + \rho_1 < 1$, see Appendix A.

A more challenging setting is feature-dependent label noise (FDLN), where the distribution of the noisy label can {\em also} depend on the feature vector. In this case, the label noise is characterized by functions $\rho_i(x) = \Pr(\Yt \ne i \, | \, Y=i, X=x)$, $i \in \{0,1\}$, which give the probability that a training label is flipped, depending on the true class label and the feature vector $x$. These two functions are potentially quite complex, and prior work has made strong assumptions on these functions or the target distribution $Q$. Thus, \citet{bootkrajang16} employs a parametric model for $\rho_0(x)$ and $\rho_1(x)$, while \citet{ghosh15} provide a class of nonconvex losses that are robust to FDLN when the Bayes Risk for $Q$ is {\em zero}.

\citet{menon18} established immunity for the probability of error criterion under the condition of symmetric FDLN, that is, $\rho_0(x) = \rho_1(x)$ for all $x$, which is a strong assumption in practice. \citet{cannings18} extend this result by establishing immunity when $\rho_0(x)$ and $\rho_1(x)$ are approximately symmetric in a certain sense, approaching perfect symmetry near the decision boundary.

\citet{menon18} make two other contributions to the study of FDLN problems. They introduce a type of annotator noise called boundary-consistent noise (BCN) wherein $\rho_0(x)$ and $\rho_1(x)$ obey certain monotonicity properties, and show that this noise model implies \textbf{(PD)}. Under BCN, they show that the area under the ROC curve (AUC) is immune to FDLN. It should be noted, however, that AUC is a criterion for ranking and not for binary classification. They also study a type of generalized linear model under BCN and show that the Isotron algorithm is consistent in this setting.

\citet{cheng17} assume that $\rho_0(x)$ and $\rho_1(x)$ are bounded by a number $< 0.5$. This seems an unlikely model for annotator noise, since one would expect $\rho_1(x) \to 1$ as $\eta_Q(x) \to 0$, and $\rho_0(x) \to 1$ as $\eta_Q(x) \to 1$. Leveraging ideas from \citet{northcutt2017rankpruning}, they describe a procedure to find a subset of examples where the label is known to be correct. Knowledge of the bounds on $\rho_0(x)$ and $\rho_1(x)$ are required as input to their algorithm. Their theory analyzes a method that requires knowledge of $\eta_Q(x)$, and a more practical algorithm requires access to, or an estimate of, the same density ratio that arises in covariate shift.

Our contribution to the study of FDLN is as follows. We are the first to establish both {\em consistency} and {\em immunity} of a learning algorithm, with respect to some optimality criterion, under a realistic nonparametric model of annotator noise (namely, \textbf{(PD)}) and under general nonparametric assumptions on the data distribution. Furthermore, our approach avoids the need to estimate $\rho_0(x)$ or $\rho_1(x)$, or the density ratio mentioned previously. 


\textbf{Other Classes of Domain Adaptation:} We mention two other types of domain adaptation. \citet{zhang13} study an assumption that is dual to \textbf{(CSPD)} in a sense. Whereas \textbf{(CSPD)} allows the marginal of $X$ to shift arbitrarily, and the conditional of $Y|X$ to shift in a monotone fashion, they allow the marginal of $Y$ to shift arbitrarily, and the conditional of $X|Y$ to undergo a location-scale shift. \citet{tasche17jmlr} introduces problems with an ``invariant density ratio," where the likelihood ratios of $P$ and $Q$ are equal. This problem is a special case of \textbf{(CSPD)} and a generalization of target shift.

\textbf{Optimality Criteria for Binary Classification}: There has been interest in recent years in cataloging different performance measures and optimality criteria for binary classification \citep{naga14nips,narasimhan14nips,kotlowski15,dembczynski17}, and establishing consistent learning rules for them. The GNP criteria are evidently a new family of optimality criteria, thus expanding this literature. \citet{tasche18} studies a different family of constrained optimization problems that also includes the CDR criterion as a special case, providing an alternate proof of Theorem \ref{thm:con} in the case of CDR. The fact that the CDR criterion is optimized by thresholding $\eta_Q$ was noted by \citet{clemencon07jmlr}, see Remark 2.

\textbf{NP Classification}: We anticipate that several existing algorithms for Neyman-Pearson classification \citep{scott05np,tong16} and similar constrained criteria extend naturally to CDR. To illustrate this point, later we present an adaptation of an algorithm of \citet{lei14}. In the reverse direction, our algorithm and analysis based on kernel logistic regression should naturally yield algorithms and analysis for Neyman-Pearson classification as well as other classification and level-set criteria.

\section{Estimators for the CDR Criterion}
\label{sec:cdr}

In this section we address consistent estimators for the optimal CDR classifier.  Our goal is to estimate the set
\begin{equation}
\label{eqn:cdrset}
G_{Q,\a}:= \{x : \eta_{Q}(x) \ge t_{Q,\a}\}
\end{equation}
where $t_{Q,\a}$ is the threshold associated to the CDR criterion at level $\a$. In other words, $Q_X(G_{Q,\a}) = \a$. Note that this assumes the optimal classifier $g_{Q,\a}$ is deterministic, which is formalized in our distributional assumptions below. Also, we view deterministic classifiers and subsets of $\sX$ interchangeably by viewing the classifier as an indicator on the subset.

For greater generality that will be needed in the context of domain adaptation, we actually consider the problem of estimating
\begin{equation}
\label{eqn:cdrset2}
G_{P,Q,\a}:=\{x : \eta_P(x) \ge t_{P,Q,\a}\}
\end{equation}
where $t_{P,Q,\a}$ is such that $Q_X(G_{P,Q,\a}) = \a$. Note that taking $P=Q$ reduces to \eqref{eqn:cdrset2} to \eqref{eqn:cdrset}.

To preview Section \ref{sec:da}, in the context of domain adaptation, $G_{P,Q,\a}$ can be estimated since we have data drawn from $P$ and $Q_X$. Furthermore, under \textbf{(CSPD)}, it is not hard to see that $G_{P,Q,\a} = G_{Q,\a}$, meaning it is possible to consistently estimate the optimal CDR classifier on the target domain.

After formalizing our distributional assumptions and the estimation problem, we present two estimators with associated convergence results. The first assumes access to a sup-norm consistent estimator of $\eta_P$, while the second uses kernel logistic regression to estimate $\eta_P$. Throughout this section we assume $\sX$ is a compact subset of $\R^d$.

\subsection{Distributional Assumptions}

In addition to \textbf{(CSPD)}, our analysis makes the following nonparametric assumptions on $P$ and $Q$. These assumptions  allow $P$ and $Q$ to be quite different from one another according to essentially any commonly used notion of distance or divergence between two distributions.

Define $F_{P,Q}(t) := Q_X(\{x: \eta_P(x) \le t\})$, the cumulative distribution function of the random variable $\eta_P(X)$ when $X \sim Q_X$. We adopt the following two assumptions:
\begin{description}
\item[\textbf{(A)}] There exists $t_{P,Q,\a} \in (0,1]$ such that
$$
Q_X(\{x: \eta_P(x) \ge t_{P,Q,\a}\}) = \a.
$$
\item[\textbf{(B)}] There exist positive constants $\delta_0, b_1, b_2$ and $\kappa$
such that for all $\delta \in [-\delta_0, \delta_0]$,
$$
b_1 |\delta|^\kappa \le |F_{P,Q}(t_{P,Q,\a} + \delta) - F_{P,Q}(t_{P,Q,\a})| \le b_2 |\delta|^\kappa.
$$
\end{description}

\textbf{(A)} ensures that randomized classifiers are not needed.
\textbf{(B)} states that $F_{P,Q}$ has local growth (in a neighborhood of $t_{P,Q,\a}$) characterized by the exponent $\kappa$, which characterizes the difficulty of the estimation problem.  The lower bound in \textbf{(B)} implies that $t_{P,Q,\a}$ is unique, while the upper bound implies that $F_{P,Q}$ is continuous at $t_{P,Q,\a}$.
Under \textbf{(A)} and \textbf{(B)}, $G_{P,Q,\a}$ is well-defined, i.e., the threshold $t_{P,Q,\a}$, which must satisfy $Q_X(G_{P,Q,\a}) = \a$, exists and is unique.

The following assumption is widely adopted in the study of covariate shift.
\begin{description}
\item[\textbf{(C)}] The support of $Q_X$ is contained in the support of $P_X$.
\end{description}
A strengthened form of this assumption is employed in the analysis of our second algorithm \citep{yu12}.
\begin{description}
\item[\textbf{(C')}] There exists $c_0 > 0$ such that $Q_X \le c_0 P_X$. Equivalently, $Q_X$ is absolutely continuous with respect to $P_X$, and $\partial Q_X / \partial P_X$ is essentially bounded by $c_0$.
\end{description}

\subsection{The Estimation Problem}

We focus on estimating $G_{P,Q,\a}$ given the following data:
\begin{align*}
(X_1,Y_1) \ldots, (X_m,Y_m) &\stackrel{iid}{\sim} P \\
X_{m+1}, \ldots, X_{m+n} &\stackrel{iid}{\sim} Q_X.
\end{align*}
The two samples are assumed to be independent of each other.
Let $\Ghat_{P,Q,\a}$ be an estimate of $G_{P,Q,\a}$. We further focus on the performance measure
$$
Q_X(\Ghat_{P,Q,\a} \Delta G_{P,Q,\a}),
$$
where $G \Delta G' := (G \backslash G') \cup (G' \backslash G)$ is the
{\em symmetric difference} of $G$ and $G'$.

According to the following result, convergence with respect to the above measure implies convergence of the objective and constraint functions for GNP criteria.

\begin{prop}
\label{prop:symdiff}
Let $g$ and $g'$ be two deterministic classifiers, and let $G=\{x \ : \ g(x) = 1\}$ and $G'=\{x \ : \ g'(x) = 1\}$ be the associated sets. For any $\eps \in [0,1]$ and any $Q$,
$$
\left| \eps B_Q(g) + (1-\eps) A_Q(g) - [\eps B_Q(g') + (1-\eps) A_Q(g')] \right| \le \left(\frac{\eps}{\pi_Q} + \frac{1-\eps}{1-\pi_Q} \right) Q_X(G \Delta G').
$$
\end{prop}
\sloppy In what follows, let $P^m$ denote
the product measure governing $(X_1,Y_1) \ldots, (X_m,Y_m)$, and $Q_X^n$
denote the product measure governing $X_{m+1}, \ldots, X_{m+n}$. We use $\Pr$ to denote the product measure $P^m \times Q_X^n$ on $(\sX \times \{0,1\})^m \times \sX^n$, which governs the combined draw of the two
samples. The goal is to show $\Pr(Q_X(\Ghat_{P,Q,\a} \Delta G_{P,Q,\a})) \to 0$ in probability as $m,n \to \infty$.





\subsection{A result based on sup-norm consistent estimation of the posterior}
\label{sec:supnorm}

The CDR criterion is sufficiently similar to NP classification and related problems that we can easily modify existing algorithms and theory to our setting. To illustrate this, we begin by establishing a consistent CDR estimator based on a sup-norm
consistent estimate of $\eta_P$. The results in this subsection translate ideas from \citet{lei14}, where a different generalization of the Neyman-Pearson criterion was considered. Let $\etahat_P$ denote an estimate, based on $(X_1,Y_1), \ldots, (X_m,Y_m)$,
of the posterior $\eta_P$ associated to the joint distribution $P$. Let $\delta_m, \tau_m$ be two sequences of positive reals numbers tending to 0.

\begin{defn}
An estimator $\etahat_P$ is $(\delta_m, \tau_m)$-accurate if $P^m( \|
\etahat_P - \eta_P \|_\infty \ge \delta_m) \le \tau_m$ as $m \to \infty$.
\end{defn}

Specific examples of $(\delta_m, \tau_m)$-accurate estimators are
provided by \citet{lei14}, with explicit rates (tending to 0) for
$\delta_m$ and $\tau_m$. In particular, \citet{audibert07} and \citet{vandegeer08} give explicit rates for local polynomial regression and $\ell_1$-penalized logistic regression, respectively. These estimators in turn yield explicit rates of
convergence in our setting. We refer the reader to \citet{lei14} for details.

{\em Remark:} $(\delta_m, \tau_m)$-accurate estimators of $\eta_P$ may require additional distributional assumptions on $P$ beyond what we have assumed so far. This is the case for the two examples mentioned above. This does not change our conclusion that $P$ and $Q$ can still be substantially different. Also, our goal in this subsection is to demonstrate an estimator with a rate of convergence, but other consistent estimators that do not require additional assumptions could also be adapted to CDR estimation.



Define $\Ghat_{P,Q,\a} = \{x : \etahat_P(x) \ge \that_{P,Q,\a}\}$, where
$\that_{P,Q,\a}$ is the $\lfloor n(1-\a) \rfloor$th smallest value among
$\{\etahat_P(X_{m+1}), \ldots, \etahat_P(X_{m+n})\}$.

\begin{thm}
\label{thm:unif}
Let $P$ and $Q$ be joint distributions, and let
$(X_1,Y_1) \ldots, (X_m,Y_m) \stackrel{iid}{\sim} P$
and $X_{m+1}, \ldots, X_{m+n} \stackrel{iid}{\sim} Q_X$.
Assume \textbf{(A)}, \textbf{(B)}, and \textbf{(C)} hold, and that $\etahat_P$ is a
$(\delta_m, \tau_m)$-accurate estimator of $\eta_P$. For each $r > 0$,
there exists a positive constant $c$ such that for $m$ and $n$ large enough,
with probability at least $1 - \tau_m - n^{-r}$ with respect
to the draw of the training data,
$$
Q_X(\Ghat_{P,Q,\a} \Delta G_{P,Q,\a^*}) \le c \left\{ \delta_m^\kappa + \left(\frac{\log n}{n} \right)^{1/2} \right\}.
$$
\end{thm}
When this result is instantiated with the $(\delta_m,\tau_m)$-accurate estimator of 
\citet{audibert07}, and $\kappa =1$, the rate above matches or is similar to known rates for related set estimation and classification problems. See \citet{lei14} for additional discussion.


\subsection{A result for kernel logistic regression}
\label{sec:klr}

In this section, we examine an estimator based on kernel logistic regression (KLR), which is perhaps a more practical estimator for $\eta_P$ than the methods mentioned in the previous subsection. Although KLR is not known to be sup-norm consistent, we are able to establish an asymptotic convergence result for our estimator based on theory developed by \citet{steinwart03}. We believe this is the first such result for a set estimator based on KLR.

%
%
%


Let $\etahat_P$ be the estimate of $\eta_P$ resulting from KLR with symmetric, positive definite kernel $k$ and
regularization parameter $\lambda$, based on $(X_i, Y_i), i=1,\ldots,m$.
That is,
$$
\etahat_P(x) = \frac1{1 + \exp(- \fhat_P(x))}
$$
where $\fhat_P$ solves
$$
\min_{f \in \sH} \ \frac{\lambda}{2} \| f \|_{\sH}^2 + \frac1{m}
\sum_{i=1}^m \log (1 + \exp(- (2Y_i - 1) f(X_i))).
$$
Here $\sH$ is a reproducing kernel Hilbert space of functions over
$\R^d$ associated to kernel $k$. Later, we will assume that $k$ is a {\em universal kernel}, which means that $\sH$ has nice approximation properties \citep{steinwart08}.

For a set $G$ define $\Qhat_X(G) = \frac1{n} \sum_{i=m+1}^{m+n} \ind{X_i \in
G}$, the empirical measure with respect to the second training sample.
Let $\alpha$ be the user-specified constant defining the CDR
criterion. Now define the empirical estimate of $t_{P,Q,\a}$, with tuning parameters
$\beta$ and $\gamma$, as
\begin{equation}
\label{eqn:that}
\that_{P,Q,\a} = \inf \{ t \, | \, \Qhat_X(\{x : \etahat_P(x) \ge t + \beta \}) \le \a + \gamma + \eps_n \}
\end{equation}
where $\eps_n = 4 (\log (n+1) / n)^{1/2}$, and define the estimator of $G_{P,Q,\a}$ to be
\begin{equation}
\label{eqn:Qhat}
\Ghat_{P,Q,\a} = \{x : \etahat_P(x) \ge \that_{P,Q,\a}\}.
\end{equation}






\begin{thm}
\label{thm:klr}
Assume \textbf{(A)}, \textbf{(B)}, and \textbf{(C')} hold. Let $k$ be a universal kernel and let $\lambda = \lambda_m$
such that $\lambda \to 0$ and $m \lambda^2 \to \infty$. For all $\eps > 0$, there exist $\beta$ and $\gamma$ such that
$$
Q_X(\Ghat_{P,Q,\a} \Delta G_{P,Q,\a}) \le \eps
$$
in probability as $m, n \to \infty$.
\end{thm}
The proof hinges on a result of \citet{steinwart03}, who effectively shows that $\widehat{\eta}_P$ is uniformly close to $\eta_P$, to arbitrary accuracy, on an event with probability tending to 1 as $m \to \infty$. We then use \textbf{(B)} to translate accuracy of $\widehat{\eta}_P$ to accuracy of the associated set estimate. The proof gives constructive choices for $\b$ and $\g$ depending on $\eps$ and the constants appearing in \textbf{(B)}. Concrete rates of convergence are not available because the same is true of the result of \citet{steinwart03} that we leverage.

This result does not show consistency of a specific algorithm, since $\beta$ and $\gamma$ depend on $\eps$. Nonetheless it demonstrates the theoretical capacity of a KLR-based estimator to deliver arbitrarily accurate estimates of $G_{P,Q,\a}$. In practice, of course, the threshold on $\widehat{\eta}_P$ would be determined in a data-driven fashion \citep{tong18npcv}.

%

%

%
%



\section{Domain Adaptation for the CDR Criterion}
\label{sec:da}

Recall that the goal of domain adaptation with the CDR criterion is to recover
$$
G_{Q,\a} = \{x : \eta_Q(x) \ge t_{Q,\a}\}
$$
given realizations of $P$ and of $Q_X$. In the previous section, we saw that it is possible to consistently estimate
$$
G_{P,Q,\a} = \{x : \eta_P(x) \ge t_{P,Q,\a}\}
$$
under assumptions \textbf{(A)}, \textbf{(B)}, and \textbf{(C)} or \textbf{(C')}.

The key insight of this paper is that under \textbf{(CSPD)}, $G_{Q,\a} = G_{P,Q,\a}$, and therefore $G_{Q,\a}$ can be consistently estimated. To see that $G_{Q,\a} = G_{P,Q,\a}$ under \textbf{(CSPD)}, simply recall the definition of \textbf{(CSPD)} which assumes the existence of a strictly increasing function $\phi: [0,1] \to [0,1]$ such that for all $x$, $\eta_P(x) = \phi(\eta_Q(x))$. Now, $G_{Q,\a} = G_{P,Q,\a}$ follows by taking $t_{Q,\a} = \phi^{-1}(t_{P,Q,\a})$. Under \textbf{(A)} and \textbf{(B)}, $t_{P,Q,\a}$ exists and is unique, and therefore the same is true of $t_{Q,\a}$.

In light of the above, we have the following:
\begin{cor}
Assume \textbf{(CSPD)}, \textbf{(A)}, \textbf{(B)}, and \textbf{(C)} (respectively, \textbf{(C')}) hold. Then the estimator of Section \ref{sec:supnorm} (resp., Section \ref{sec:klr}) satisfies the conclusion of Theorem \ref{thm:unif} (resp., Theorem \ref{thm:klr}), where now the set being estimated is $G_{Q,\a}$.
\end{cor}

\section{Conclusions}

We have introduced a family of generalized Neyman-Pearson optimality criteria, and shown that a member of this family, the controlled discovery rate criterion, is immune to domain adaptation under the model of covariate shift with posterior drift. Compared with prior work on domain adaptation, we do not require that the source and target distributions be close in some sense in order to obtain optimal performance on the source domain. With respect to prior work on covariate shift, our approach does not require estimating a density ratio, and in fact allows the density ratio to be unbounded under condition \textbf{(C)}. Comparing to the literature on feature-dependent label noise, ours is the first work to establish consistency/immunity under a general and flexible model for annotator noise, without requiring knowledge of the specific annotator noise model. These results are enabled by consideration of an optimality criterion different from the usual ones based on expected loss.

\acks{The author was supported by NSF Grants No. 1422157 and 1838179.}

\bibliography{alt}

\appendix

\section{Immunity}

This appendix provides supplemental details and observations pertaining to immunity.

The immunity of an optimality criterion with respect to a class of domain adaptation problems is formally defined as follows. We distinguish between the {\em inductive} setting, where the learner has access only to labeled data from $P$, and the {\em semi-supervised} setting, where the learner has an additional unlabeled sample from $Q_X$. Let ${\cal P}$ be some class of distributions of interest, e.g., all distributions on $\sX$. A class of domain adaptation problems is a subset ${\cal D} \subseteq {\cal P} \times {\cal P}$ where, for $(P,Q) \in {\cal D}$, $P$ is the source domain and $Q$ the target. At times we express a distribution $Q$ as the pair $(\eta_Q, Q_X)$. The classifier (or set of classifiers) optimizing an optimality criterion for distribution $Q$ is denoted $\textsc{OPT}(Q)$ in the inductive case, and $\textsc{OPT}(\eta_Q, Q_X)$ in the semi-supervised case. 
We say that an optimality criterion is {\em immune} to ${\cal D}$ if, for all $(P,Q) \in {\cal D}$, $\textsc{OPT}(Q) = \textsc{OPT}(P)$ in the inductive setting, or $\textsc{OPT}(\eta_Q,Q_X) = \textsc{OPT}(\eta_P,Q_X)$ in the semi-supervised setting. 
Except for our result on the CDR criterion, all of the immunity results mentioned in Section \ref{sec:related} are for the inductive setting.


To see that LDLN is a special case of \textbf{(PD)} provided $\rho_0 + \rho_1 < 1$, observe
\begin{align*}
\eta_P(x) &= \Pr(\Yt=1 \mid X=x) \\
&= \Pr(\Yt=1 \mid Y=1, X=x) \Pr(Y=1 \mid X=x) \\
& \qquad + \Pr(\Yt=1 \mid Y=0, X=x) \Pr(Y=0 \mid X=x) \\
&=\Pr(\Yt=1 \mid Y=1) \Pr(Y=1 \mid X=x) + \Pr(\Yt=1 \mid Y=0) \Pr(Y=0 \mid X=x) \\
&= (1-\rho_1) \eta_Q(x) + \rho_0 (1-\eta_Q(x)) \\
&= (1 - \rho_0 - \rho_1) \eta_Q(x) + \rho_0.
\end{align*}

We also note that \textbf{(CSPD)} is preserved by composition of domain adaptations, because the composition of strictly increasing functions is strictly increasing. For example, consider distributions $P$, $Q$, and $R$. Let ${\cal D}$ be the set of $(P,R)$ such that there exists $Q$ for which $Q$ is related to $R$ by target shift, and $P$ is generated from $Q$ by LDLN (with $\rho_0 + \rho_1 < 1$). Then there exist strictly increasing $\phi_1$ and $\phi_2$ such that, for all $x$, $\eta_P(x) = \phi_1(\eta_Q(x))$ and $\eta_Q(x) = \phi_2(\eta_R(x))$. Thus $\eta_P(x) = \phi(\eta_R(x))$ where $\phi = \phi_1 \circ \phi_2$, which is strictly increasing, and therefore ${\cal D}$ satisfies \textbf{(CSPD)}.

The focus of the paper has been immunity of the CDR criterion to \textbf{(CSPD)} in the semi-supervised setting. We note that 
the CDR criterion is also immune to \textbf{(PD)} in the {\em inductive} setting. Since $P_X = Q_X$ under \textbf{(PD)}, $Q_X$ is already estimable through the data drawn from $P$, and an unlabeled sample from $Q_X$ is not needed. Indeed, all of the results for \textbf{(CSPD)} in the semi-supervised setting could also be stated for \textbf{(PD)} in the inductive setting.

Finally, we remark that a subset of GNP criteria (namely, when $\theta_0 = 0$) are immune to a subclass of \textbf{(PD)} corresponding to {\em one-sided} feature-dependent label noise. In particular, define the domain adaptation class
\begin{description}
\item[(PD')] $P_X = Q_X$, $\rho_1 \equiv 0$ and there exists a strictly increasing function $\psi$ such that $\rho_0(x) = \psi(\eta_Q(x))$ for all $x$.
\end{description}
Under \textbf{(PD')}, true labels of 1 are never corrupted to become 0. Furthermore we have the following.
\begin{lemma}
\textbf{(PD')} implies \textbf{(PD)}
\end{lemma}
\begin{proof}
We need to show that $\eta_P(x)$ is a strictly increasing function of $\eta_Q(x)$. For a posterior $\eta(x)$, define $\bar{\eta}(x) = 1 - \eta(x)$. Arguing as we did previously, under \textbf{(PD')},
\begin{align*}
\bar{\eta}_P(x) &= \Pr(\Yt=0 \mid X=x) \\
&= \Pr(\Yt=0 \mid Y=1, X=x) \Pr(Y=1 \mid X=x) \\
& \qquad + \Pr(\Yt=0 \mid Y=0, X=x) \Pr(Y=0 \mid X=x) \\
&=\rho_1(x) (1- \bar{\eta}_Q(x)) + (1 - \rho_0(x)) \bar{\eta}_Q(x) \\
&= (1 - \rho_0(x) - \rho_1(x)) \bar{\eta}_Q(x) + \rho_1(x) \\
&= (1 - \rho_0(x)) \bar{\eta}_Q(x) + \rho_1(x) \\
&= (1 - \rho_0(x)) \bar{\eta}_Q(x).
\end{align*}
This implies that 
$$
\eta_P(x) = 1 - (1 - \rho_0(x)) (1-\eta_Q(x)).
$$
The result now follows.
\end{proof}

Then {\em all} GNP criteria with $\theta_0 = 0$ are immune to \textbf{(PD')} in the inductive setting. This follows by similar reasoning as for CDR. First, with $\theta_0=0$, the constraint in the GNP criterion depends only on $Q_0$, and $Q_0 = P_0$ because $\rho_1(x) \equiv 0$. Second, $\eta_P$ and $\eta_Q$ are monotonically equivalent. Therefore, the level set of $\eta_P$ with $P_0$-measure $\a$ is also the level set of $\eta_Q$ with $Q_0$-measure $\a$.

\section{Proofs}

This appendix contains the proofs.

\subsection{Proof of Theorem \ref{thm:con}}

Denote
\begin{align}
\tilde{q}_1(x) := \theta_1 q_1(x) + (1-\theta_1) q_0(x), \label{eqn:con1} \\
\tilde{q}_0(x) := \theta_0 q_1(x) + (1-\theta_0) q_0(x). \label{eqn:con0}
\end{align}
Note that $\tilde{q}_1(x)$ and $\tilde{q}_0(x)$ are densities for the distributions $\tilde{Q}_1 := \theta_1 Q_1 + (1-\theta_1) Q_0$ and $\tilde{Q}_0 := \theta_0 Q_1 + (1-\theta_0) Q_0$, respectively. Viewing these as the alternative and null distributions in a hypothesis testing problem, the power and size of a classifier $g$ are
\begin{align*}
\tilde{B}_Q(g) := \theta_1 B_Q(g)  + (1-\theta_1) A_Q(g)  \\
\tilde{A}_Q(g) := \theta_0 B_Q(g) + (1-\theta_0) A_Q(g).
\end{align*}
Thus, the optimization problem in \eqref{opt:con} is equivalent to maximizing the power $\tilde{B}_Q(g)$, subject to the constraint that the size $\tilde{A}_Q(g) \le \a$. By the Neyman-Pearson lemma, the optimal classifier has the form
$$
g_\a(x) = \bcase
1, & \tilde\Lambda(x) > \lambda_\a, \\
q_\a, & \tilde\Lambda(x) = \lambda_\a, \\
0, & \tilde\Lambda(x) < \lambda_\a.
\ecase
$$
where $\tilde\Lambda(x) = \tilde{q}_1(x)/\tilde{q}_0(x)$, and $\lambda_\a>0$ and $q_\a \in [0,1)$ are uniquely determined by
$$
\tilde{Q}_0(\tilde\Lambda(X) < \lambda_\a) + q_\a \tilde{Q}_0(\tilde\Lambda(X) = \lambda_\a) = \a.
$$

Next, we apply Proposition 1 of \citet{blanchard16ejs} which we restate in our notation for convenience. (In their notation, $\pi_0 = 1-\theta_1$, $\pi_1 = \theta_0$.)
\begin{lemma}
Let $q_0$ and $q_1$ be probability density functions, let $0 \le \theta_0 < \theta_1 \le 1$, and let $\tilde{q}_1$ and $\tilde{q}_0$ be as in \eqref{eqn:con1}-\eqref{eqn:con0}. For all $\gamma \ge 0$ and all $x$ such that $q_0(x) > 0$,
$$
\frac{q_1(x)}{q_0(x)} > \gamma \iff \frac{\tilde{q}_1(x)}{\tilde{q}_0(x)} > \lambda,
$$
where
\begin{equation}
\label{eqn:lamgam}
\lambda = \frac{1-\theta_1 + \gamma \theta_1}{1 - \theta_0 + \gamma \theta_0}.
\end{equation}
\end{lemma}
The result states that the ``pure" and ``contaminated" likelihood ratios are monotonically equivalent.

Before applying this result, we make the following observations.  First, by inspecting \eqref{eqn:lamgam}, as $\gamma$ varies from $0$ to $\infty$, $\lambda$ varies between its extremes,
$$
\frac{1-\theta_1}{1-\theta_0} \le \lambda \le \frac{\theta_1}{\theta_0}.
$$
Second, these extremes also bound the range of the contaminated likelihood ratio, which is evident from the expression
$$
\tilde\Lambda(x) = \frac{\theta_1 q_1(x) + (1-\theta_1) q_0(x)}{\theta_0 q_1(x) + (1-\theta_0) q_0(x)}
= \frac{\theta_1 \frac{q_1(x)}{q_0(x)} + 1 - \theta_1}{\theta_0 \frac{q_1(x)}{q_0(x)} + 1 - \theta_0}.
$$
Third, given $\lambda$ in this range, one can solve for $\gamma$,
$$
\gamma = \frac{\lambda(1-\theta_0) - (1- \theta_1)}{\theta_1 - \lambda \theta_0} \in [0,\infty].
$$
Putting these observations together, $\lambda_\a$ must satisfy $\frac{1-\theta_1}{1-\theta_0} \le \lambda_\a \le \frac{\theta_1}{\theta_0}$, and therefore
$$
g_\a(x) = \bcase
1, & \Lambda(x) > \gamma_\a, \\
q_\a, & \Lambda(x) = \gamma_\a, \\
0, & \Lambda(x) < \gamma_\a,
\ecase
$$
where $\Lambda(x) = q_1(x)/q_0(x)$ and
$$
\gamma_\a = \frac{\lambda_\a(1-\theta_0) - (1- \theta_1)}{\theta_1 - \lambda_\a \theta_0} \in [0,\infty].
$$
Finally, by
$$
\eta_Q(x) = \frac{\pi_Q q_1(x)}{\pi_Q q_1(x) + (1-\pi_Q) q_0(x)} = \frac{\pi_Q \Lambda(x)}{\pi_Q \Lambda(x) + 1 - \pi_Q},
$$
we know that $\eta_Q(x)$ is monotonically equivalent to $\Lambda(x)$. This completes the proof.

\subsection{Proof of Proposition \ref{prop:symdiff}}

By the triangle inequality,
\begin{align*}
\left| \eps B_Q(g) + (1-\eps) A_Q(g) - [\eps B_Q(g') + (1-\eps) A_Q(g')] \right| \\ \le
\eps \left|B_Q(g) - B_Q(g') \right| + (1-\eps) \left|A_Q(g) - A_Q(g') \right|.
\end{align*}
We claim that $|B_Q(g) - B_Q(g')| \le Q_1(G \Delta G')$. To see this, observe
\begin{align*}
B_Q(g) - B_Q(g') &= Q_1(G) - Q_1(G') \\
&= Q_1(G \backslash G') - Q_1(G' \backslash G) \\
&\le Q_1(G \backslash G') + Q_1(G' \backslash G) \\
&= Q_1(G \Delta G').
\end{align*}
A similar argument shows that $B_Q(g) - B_Q(g') \ge -Q_1(G \Delta G')$ which establishes the claim.

Similarly, it can be shown that $|A_Q(g) - A_Q(g')| \le Q_0(G \Delta G')$.

Since $Q_X = \pi_Q Q_1 + (1-\pi_Q) Q_0$, we know $Q_X \ge \pi_Q Q_1$ and $Q_X \ge (1-\pi_Q) Q_0$ and therefore $Q_1 \le \frac1{\pi_Q} Q_X$ and $Q_0 \le \frac1{1-\pi_Q} Q_X$. Combining the above facts establishes the result.

\subsection{Proof of Theorem \ref{thm:unif}}

Since the support of $Q$ is contained in the support of $P$, $\etahat_P$
is $(\delta_m, \tau_m)$-accurate on the support of $Q$.

Let $\Fhat_{P,Q}(t)$ be the empirical CDF of the random variable $\eta_P(X)$,
$X \sim Q_X$, based on
$X_{m+1}, \ldots, X_{m+n}$. For $r > 0$, introduce the event
$$
E_r = \Big\{ \| \etahat_P - \eta_P \|_\infty \le \delta_m, \sup_t |F_{P,Q}(t) -
\Fhat_{P,Q}(t)| \le c_r(\log n/n)^{1/2} \Big\}.
$$
By the DKW inequality \citep{massart90}, there exists $c_r$ such that
$E_r$ occurs with probability at least $1 - \tau_m - n^{-r}$.

{\em Remark:} The advantage of having the theorem hold for arbitrary $r> 0$ is that for some estimators, e.g., the $\ell_1$-penalized logistic regression estimator studied by \citet{vandegeer08}, $r$ needs to be sufficiently large for the estimator to be $(\delta_m, \tau_m)$-accurate with specific rates for $\delta_m$ and $\tau_m$.

The proof hinges on the following lemma.
\begin{lemma}
There exists $c_{r,\kappa} > 0$ such that for $m$ and $n$ large enough,
on $E_r$,
$$
|\that_{P,Q,\a} - t_{P,Q,\a}| \le \delta_m + c_{r,\kappa} \left( \frac{\log
n}{n} \right)^{1/2\kappa}.
$$
\end{lemma}
\begin{proof}
Introduce the sets $L_P(t) = \{x : \eta_P(x) \le t\}$ and $\Lhat_P(t) = \{x :
\etahat_P(x) \le t\}$. Observe that for any $t \in [0,1]$,
$$
\Qhat_X(\Lhat_P(t)) \le \Qhat_X(L_P(t + \delta_m)) = \Fhat_{P,Q}(t + \delta_m)
\le F_{P,Q}(t + \delta_m) + c_r \left( \frac{\log n}{n} \right)^{1/2}.
$$
Now let $t_{P,Q,\a}' := t_{P,Q,\a} - \delta_m - \{2c_r b_1 (\log
n/n)^{1/2}\}^{1/\kappa}$, where $b_1$ is from \textbf{(B)}. For $m$ and $n$
large enough, we have $\delta_m + \{2c_r b_1 (\log
n/n)^{1/2}\}^{1/\kappa} \le t_{P,Q,\a}$ (so that $t_{P,Q,\a}' \in [0,1]$),
$1/n < c_r(\log n / n)^{1/2}$, and $\{2c_r b_1 (\log
n/n)^{1/2}\}^{1/\kappa} \le \delta_0$ where
$\delta_0$ is from \textbf{(B)}. It follows that
\begin{align*}
\Qhat_X(\Lhat_P(t_{P,Q,\a}')) &\le F_{P,Q}(t_{P,Q,\a} - \{2c_r b_1 (\log
n/n)^{1/2}\}^{1/\kappa}) +  c_r \left( \frac{\log n}{n} \right)^{1/2} \\
& \le F_{P,Q}(t_{P,Q,\a}) - c_r \left( \frac{\log n}{n} \right)^{1/2} \\
& = 1 - \a - c_r \left( \frac{\log n}{n} \right)^{1/2} \\
& < 1 - \a - n^{-1} \\
& \le \lfloor n(1-\a)\rfloor /n \\
& \le \Qhat_X(\Lhat_P(\that_{P,Q,\a})),
\end{align*}
where the second inequality follows from \textbf{(B)}. It follows that
$\that_{P,Q,\a} \ge t_{P,Q,\a}' = t_{P,Q,\a} - \delta_m - c^1_{r,\kappa} (\log
n/n)^{1/2\kappa}$ where $ c^1_{r,\kappa} = \{2c_r b_1\}^{1/\kappa}$.

The reverse inequality is similar with one slight change, in that we redefine
$L_P(t) = \{x : \eta_P(x) < t\}$ and $\Lhat_P(t) = \{x :
\etahat_P(x) < t\}$. Similar to before, for any $t \in [0,1]$,
$$
\Qhat_X(\Lhat_P(t)) \ge \Qhat_P(L(t - \delta_m)) = \Fhat_{P,Q}(t - \delta_m)
\ge F_{P,Q}(t - \delta_m) - c_r \left( \frac{\log n}{n} \right)^{1/2}.
$$
Now let $t'_{P,Q,\a} := t_{P,Q\a} + \delta_m + \{2c_r b_2 (\log
n/n)^{1/2}\}^{1/\kappa}$, where $b_2$ is from \textbf{(B)}. For $m$ and $n$
large enough, we have $\delta_m + \{2c_r b_2 (\log
n/n)^{1/2}\}^{1/\kappa} \le 1- t_{P,Q,\a}$ (so that $t'_{P,Q,\a} \in [0,1]$),
$1/n < c_r(\log n / n)^{1/2}$, and $\{2c_r b_2 (\log
n/n)^{1/2}\}^{1/\kappa} \le \delta_0$ where
$\delta_0$ is from \textbf{(B)}. It follows that
\begin{align*}
\Qhat_X(\Lhat_P(t'_{P,Q,\a})) &\ge F_{P,Q}(t_{P,Q,\a} + \{2c_r b_2 (\log
n/n)^{1/2}\}^{1/\kappa}) -  c_r \left( \frac{\log n}{n} \right)^{1/2} \\
& \ge F_{P,Q}(t_{P,Q,\a}) + c_r \left( \frac{\log n}{n} \right)^{1/2} \\
& = 1 - \a + c_r \left( \frac{\log n}{n} \right)^{1/2} \\
& > 1 - \a + n^{-1} \\
& \ge \lfloor n(1-\a)\rfloor /n \\
& \ge \Qhat_X(\Lhat_P(\that_{P,Q,\a}))
\end{align*}
where the second inequality follows from \textbf{(B)}. The modified
definitions of $L_P$ and $\Lhat_P$ are needed in the final step. It follows
that $\that_{P,Q,\a} \le t'_{P,Q,\a} = t_{P,Q,\a} + \delta_m + c^2_{r,\kappa} (\log
n/n)^{1/2\kappa}$ where $ c^2_{r,\kappa} = \{2c_r b_2\}^{1/\kappa}$.

The result now follows by combining the above inequalities and taking
$c_{r,\kappa} = \max\{c^1_{r,\kappa},c^2_{r,\kappa}\}$.
\end{proof}

To prove the theorem, observe that on $E_r$,
\begin{eqnarray*}
Q_X(\Ghat_{P,Q,\a} \backslash G_{P,Q,\a}) &=& Q_X(\etahat_P(X) \ge \that_{P,Q,\a}, \eta_P(X) < t_{P,Q,\a}) \\
&\le& Q_X\left\{ t_{P,Q,\a} - \delta_m - c_{r,\kappa}\left( \frac{\log n}{n} \right)^{1/2\kappa} < \eta_P(X) < t_{P,Q,\a} \right\} \\
&=& F_{P,Q}(t_{P,Q,\a}) - F_{P,Q}\left\{ t_{P,Q,\a} - \delta_m - c_{r,\kappa}\left( \frac{\log n}{n} \right)^{1/2\kappa} \right\} \\
&\le& b_2 \left\{\delta_m + c_{r,\kappa}\left( \frac{\log n}{n} \right)^{1/2\kappa} \right\}^{\kappa} \\
&\le& 2^\kappa b_2 \left\{\delta_m^\kappa + c_{r,\kappa}^\kappa \left( \frac{\log n}{n} \right)^{1/2} \right\},
\end{eqnarray*}
where the next-to-last inequality follows from \textbf{(B)} and holds when
$m$ and $n$ are large enough that $2\delta_m + c_{r,\kappa}(\log n/n
)^{1/2\kappa} \le \delta_0$. The other term is handled similarly.

\subsection{Proof of Theorem \ref{thm:klr}}


The following result follows from a result of \citet{steinwart03}.

\begin{lemma}
\label{lem:steinwart}
Let $k$ be a universal kernel and let $\lambda = \lambda_m$ such that $\lambda \to 0$ and $m \lambda^2 \to \infty$. For all $\beta, \gamma, \nu \in (0,1)$, for $m$ sufficiently large,
$$
P_X(\{x : |\eta_P(x) - \etahat_P(x)| \ge \beta\}) \le \g
$$
with probability at least $1-\nu$ with respect to the draw of $(X_i, Y_i), i=1,\ldots,m$.
\end{lemma}
In words, the $P_X$-measure of the set where $\etahat_P$ deviates from $\eta_P$ by more than $\beta$ can be made arbitrarily small, with arbitrarily high probability, by taking $m$ large enough.

\begin{proof}
Denote
$$
E_m(\beta) = \{x : |\eta_P(x) - \etahat_P(x)| \ge \beta\}.
$$
Define $f_P(x) = \log (1-\eta_P(x))/\eta_P(x)$ and observe that $\eta_P(x) = \tau(f_P(x))$ and $\etahat_P(x) = \tau(\fhat_P(x))$, where $\tau(f) = (1 + \exp(-f))^{-1}$. Also define
$$
F_m(\beta) = \{x : |f_P(x) - \fhat_P(x)| \ge \beta\}.
$$
Notice that $E_m(\beta) \subseteq F_m(\beta)$ because $\tau$ is 1-Lipschitz. The result now follows from Theorem 35 of \citet{steinwart03} (see also Theorem 22 and Remark 24).
\end{proof}

For any $\beta, \gamma \in (0,1)$, let $\Theta_m(\beta, \g)$ be the event on
which $P_X(\{x : |\etahat_P(x) - \eta_P(x)| \ge \beta\}) \le \g$. By Lemma
\ref{lem:steinwart}, $\Pr(\Theta_m(\beta,\gamma))$ can be made arbitrarily
close to 1 by taking $m$ sufficiently large.

Now consider the family of sets $\sC= \{C_t \, | \, t \ge 0\}$ where
$C_t = \{x : \etahat(x) \ge t\}$. This family has a shatter coefficient
$S(\sC,n) = n+1$. By the VC inequality \citep{dgl},
\begin{equation}
\label{eqn:vc}
|Q_X(C_t) - \Qhat_X(C_t)| \le \sqrt{\frac{8 (\log S(\sC,n) + \log n)}{n}} \le 4\sqrt{\frac{\log (n+1)}{n}} = \eps_n
\end{equation}
with probability at least $1 - 1/n$. This follows by applying the VC inequality to the conditional distribution of
$X_{m+1}, \ldots, X_{m+n}$ given $(X_1,Y_1), \ldots, (X_m,Y_m)$, and then marginalizing over $(X_1,Y_1), \ldots, (X_m,Y_m)$.

Let $\Omega_n$ denote the event on which the bound of \eqref{eqn:vc}
holds. Thus, $\Pr(\Omega_n) \ge 1 - 1/n$.


\begin{lemma}
\label{lem:thatklr}
Fix $\beta, \g > 0$, and assume \textbf{(A)}, \textbf{(B)}, and \textbf{(C)} hold. 
On the event $\Theta_m(\beta, \g/c_0) \cap \Omega_n$
$$
\that_{P,Q,\a} \le t_{P,Q,\a}.
$$
Furthermore, if $\g$ satisfies $(3 \g / b_1)^{1/\kappa} < \delta_0$, then for $n$ sufficiently large, on the event $\Theta_m(\beta,\g/c_0) \cap \Omega_n$ 
$$
t_{P,Q,\a} - \that_{P,Q,\a} \le 2\beta + (3(\g + \eps_n)/b_1)^{1/\kappa}.
$$
\end{lemma}
\begin{proof}
Assume $\Theta_m(\beta,\g/c_0) \cap \Omega_n$ occurs.
Recall
$$
\that_{P,Q,\a} := \inf\{t \, | \, \Qhat_X(\{x : \etahat_P(x) \ge t + \beta\}) \le \a + \gamma + \eps_n\},
$$
and
$$
t_{P,Q,\a} := \inf\{t \, | \, Q_X(\{x : \eta_P(x) \ge t\}) \le \a \}.
$$
To see that $\that_{P,Q,\a} \le t_{P,Q,\a}$ on $\Theta_m(\beta,\g/c_0) \cap \Omega_n$, from the definition of $t_{P,Q,\a}$ we have $Q_X(\{x : \eta_P(x) \ge t_{P,Q,\a}\}) \le \a$. By $\Theta_m(\beta,\g/c_0)$ and \textbf{(C')}, it follows that $Q_X(\{x : \etahat_P(x) \ge t_{P,Q,\a} + \beta\}) \le \a + \g$, and by $\Omega_n$, we have that $\Qhat_X(\{x : \etahat_P(x) \ge t_{P,Q,\a} + \beta\}) \le \a + \g + \eps_n$. The result follows by definition of $\that_{P,Q,\a}$.

For the reverse direction, let $\g$ be small enough such that $(3 \g/b_1)^{1/\kappa} < \delta_0$.
Assume $n$ is large enough that $(3(\g +\eps_n)/b_1)^{1/\kappa} \le \delta_0$.

Let $\tilde{t}_{P,Q,\a} := t_{P,Q,\a} - q$, where $q = 2\beta + (3(\g + \eps_n)/b_1)^{1/\kappa}$.
On $\Theta_m(\beta,\g) \cap \Omega_n$, we have that
\begin{eqnarray*}
\Qhat_X(\{x: \etahat_P(x) \ge \tilde{t}_{P,Q,\a} + \beta \}) &=& \Qhat_X(\{x: \etahat_P(x) \ge t_{P,Q,\a}  - q + \beta \}) \\
&\ge& Q_X(\{x: \etahat_P(x) \ge t_{P,Q,\a}  - q + \beta \}) - \eps_n \\
&\ge& Q_X(\{x: \eta_P(x) \ge t_{P,Q,\a}  - q + 2\beta \}) - \g - \eps_n \\
&=& 1 - F_{P,Q}(t_{P,Q,\a}  - q + 2\beta) - \g - \eps_n \\
&\ge& \alpha + 2\g + 2\eps_n,
\end{eqnarray*}
where the last step follows from \textbf{(B)} and $F_{P,Q}(t_{P,Q,\a}) = \a$. We conclude that $\that_{P,Q,\a} \ge \tilde{t}_{P,Q,\a}$, and therefore $t_{P,Q,\a} - \that_{P,Q,\a} \le q = 2\beta + (3(\g + \eps_n)/b_1)^{1/\kappa}$.

\end{proof}

To prove the theorem, let $\eps > 0$ and $\xi > 0$. We will show that for $\beta, \g$ sufficiently small, $\Pr(Q_X(\Ghat_{P,Q,\a} \Delta G_{P,Q,\a}) \le \eps) \ge 1-\xi$ for $m$ and $n$ sufficiently large. Thus, select $\beta$ and $\gamma$ such that (i) $3 \beta + (3\g/b_1)^{1/\kappa} < \delta_0$, and (ii) $b_2 \beta^\kappa + b_2 (3\beta+(3\g/b_1)^{1/\kappa})^\kappa + 2\gamma < \eps$.

Having fixed $\b$ and $\g$, let $n$ be sufficiently large such that (i) the conclusion of Lemma \ref{lem:thatklr} holds, (ii) $1/n < \xi/2$, and (iii) $b_2 \beta^\kappa + b_2 (3\beta+(3(\g+\eps_n)/b_1)^{1/\kappa})^\kappa + 2\gamma < \eps$. Also let $m$ be sufficiently large that $\Pr(\Theta_m(\beta,\g/c_0)) \ge 1 - \xi/2$. Thus, $\Theta_m(\b,\g/c_0) \cap \Omega_n$ occurs with probability at least $1 - \xi$.

Observe
\begin{align*}
Q_X(G_{P,Q,\a} \Delta \Ghat_{P,Q,\a}) &= Q_X(\eta_P(X) \ge t_{P,Q,\a}, \etahat_P(X) < \that_{P,Q,\a}) \\
& \qquad + Q_X(\eta_P(X) < t_{P,Q,\a}, \etahat_P(X) \ge \that_{P,Q,\a}) .
\end{align*}
The first term may be bounded on $\Theta_m(\b,\g) \cap \Omega_n$ as
\begin{eqnarray*}
Q_X(\eta_P(X) \ge t_{P,Q,\a}, \etahat_P(X) < \that_{P,Q,\a}) &\le& Q_X(\eta_P(X) \ge t_{P,Q,\a}, \etahat_P(X) < t_{P,Q,\a}) \\
&\le& Q_X(t_{P,Q,\a} \le \eta_P(X) \le t_{P,Q,\a} + \beta) + \gamma \\
&=&F_{P,Q}(t_{P,Q,\a} + \beta) - F_{P,Q}(t_{P,Q,\a}) + \gamma \\
&\le& b_2 \beta^\kappa + \gamma,
\end{eqnarray*}
where the first step follows from Lemma \ref{lem:thatklr}, the second from $\Theta_m(\b,\g/c_0)$, and the last from \textbf{(B)}.

As for the second term, let $q = 2\beta + (3(\g + \eps_n)/b_1)^{1/\kappa}$. Then
\begin{eqnarray*}
Q_X(\etahat_P(X) \ge \that_{P,Q,\a}, \eta_P(X) < t_{P,Q,\a}) &\le& Q_X(\etahat_P(X) \ge t_{P,Q,\a} - q, \etat_P(X) < t_{P,Q,\a}) \\
&\le& Q_X(t_{P,Q,\a} - q - \beta \le \eta_P(X) \le t_{P,Q,\a}) + \gamma \\
&=&F_{P,Q}(t_{P,Q,\a} - q - \beta) - F_{P,Q}(t_{P,Q,\a}) + \g\\
&\le& b_2 (3\beta+(3(\g + \eps_n)/b_1)^{1/\kappa})^\kappa + \gamma.
\end{eqnarray*}
with similar reasoning as the first case. The result now follows from the selected properties of $\b, \g, m$ and $n$.

\end{document}